\documentclass[twoside]{article}

\usepackage[accepted]{aistats2024}
\usepackage[algo2e, linesnumbered, vlined,ruled]{algorithm2e}
\usepackage{amssymb}
\usepackage{amsmath} 
\usepackage{amsfonts} 
\usepackage{bbm} 
\usepackage{amsthm} 
\usepackage{xspace}
\usepackage{times}
\usepackage{algorithm}
\usepackage{algorithmic}

\theoremstyle{plain}
\newtheorem{theorem}{Theorem}[section]

\newtheorem{lemma}{Lemma}[section]

\theoremstyle{definition}

\newtheorem{assumption}{Assumption}[section]
\theoremstyle{remark}
\newtheorem{remark}[theorem]{Remark}

\DeclareMathOperator*{\spn}{sp}
\DeclareMathOperator*{\val}{val}

\newcommand{\calS}{{\mathcal{S}}}
\newcommand{\calA}{{\mathcal{A}}}
\newcommand{\calT}{{\mathcal{T}}}

\newcommand{\calF}{{\mathcal{F}}}

\newcommand{\pssg}{\texttt{PSRL-ZSG}\xspace}
\newcommand{\ucsg}{\texttt{UCSG}\xspace}

\newcommand{\E}{\mathbb{E}}

\newcommand{\order}{\mathcal{O}}
\newcommand{\otil}{\widetilde{\mathcal{O}}}

\renewcommand{\S}{\mathcal{S}}

\newcommand{\A}{\mathcal{A}}
\renewcommand{\P}{\mathbb{P}}
\newcommand{\1}{\mathbbm{1}}
\newcommand{\suchthat}{\;\ifnum\currentgrouptype=16 \middle\fi|\;}

\usepackage[round]{natbib}

\begin{document}

%

%

\twocolumn[

\aistatstitle{A Bayesian Learning Algorithm for Unknown Zero-sum Stochastic Games with an Arbitrary Opponent}

\aistatsauthor{ Mehdi Jafarnia-Jahromi \And  Rahul Jain  \And Ashutosh Nayyar }

\aistatsaddress{Google DeepMind \And  USC and Google Research \And University of Southern California} ]


\begin{abstract}
In this paper, we propose Posterior Sampling Reinforcement Learning for Zero-sum Stochastic Games (\pssg), the first online learning algorithm that achieves Bayesian regret bound of $\otil(HS\sqrt{AT})$ in the infinite-horizon zero-sum stochastic games with average-reward criterion. Here $H$ is an upper bound on the span of the bias function, $S$ is the number of states, $A$ is the number of joint actions and $T$ is the horizon. We consider the online setting where the opponent can not be controlled and can take any arbitrary time-adaptive history-dependent strategy. Our regret bound improves on the best existing regret bound of $\otil(\sqrt[3]{DS^2AT^2})$ by \cite{wei2017online} under the same assumption and matches the theoretical lower bound in $T$.
\end{abstract}


\section{INTRODUCTION}
\label{sec: intro}
Recent advances in playing the game of Go \citep{silver2017mastering} and Starcraft \citep{vinyals2019grandmaster} have proved the capability of \textit{self-play} in achieving super-human performance in competitive reinforcement learning (competitive RL) \citep{crandall2005learning}, a special case of multi-agent RL where each player tries to maximize its own reward. These self-play algorithms are able to learn through repeatedly playing against themselves and update their policy based on the observed trajectory in the absence of human supervision. Despite the empirical success, the theoretical understanding of these algorithms is limited and is significantly more challenging than the single-agent RL due to its multi-agent nature. 

Self-play can be considered as a special case of offline competitive RL  where the learning algorithm controls both the agent and the opponent during the learning process \citep{bai2020provable,bai2020near}. In the more general and sophisticated online learning case, the opponent can take arbitrary history-dependent strategies and the agent has no control on the opponent during the learning process \citep{wei2017online,xie2020learning,tian2021online}.

In this paper, we consider the online learning setting where the agent learns against an arbitrary opponent who can follow a \textit{time-variant history-dependent policy and can switch its policy at any time}. We consider infinite-horizon two-player zero-sum stochastic games (SGs) with the average-reward criterion. At each time, both players determine their actions simultaneously upon observing the state. The reward and the probability distribution of the next state is then determined by the chosen actions and the current state. The players' payoffs sum to zero, i.e., the reward of one player (agent) is exactly the loss of the other player (opponent). The agent's goal is to maximize its cumulative reward while the opponent tries to minimize the total loss. The problem of designing learning algorithms that can learn against arbitrary opponents is a significant open issue. There is extensive literature on designing and analyzing algorithms that learn against opponents in such a manner that they together converge to an equilibrium of the underlying game. In such cases however, the opponent is not free to choose any learning or non-learning strategy that they want, a significant limitation in their practical use.

We propose \textit{Posterior Sampling Reinforcement Learning algorithm for Zero-sum Stochastic Games} (\pssg), a learning algorithm that achieves $\otil(HS\sqrt{AT})$ Bayesian regret bound. Here $H$ is an upper bound on the bias-span, $S$ is the number of states, $A$ is the size of all possible action pairs for both players, $T$ is the horizon, and $\otil$ hides logarithmic factors. The best existing result in this setting is achieved by \ucsg algorithm \citep{wei2017online} which obtains a regret bound of $\otil(\sqrt[3]{DS^2AT^2})$ where $D \geq H$ is the diameter of the SG. As stochastic games generalize Markov Decision Processes (MDPs), our regret bound is optimal (except for logarithmic factors) in $T$ due to the lower bound provided by \cite{jaksch2010near}.

\subsection*{Related Literature}
SG was first formulated by \cite{shapley1953stochastic}. A large body of work focuses on finding the Nash equilibria in SGs with known transition kernel \citep{littman2001friend, hu2003nash,hansen2013strategy}, or learning with a generative model \citep{jia2019feature,sidford2020solving,zhang2020model} to simulate the transition for an arbitrary state-action pair. In these cases no exploration is needed. 

There is a long line of research on exploration and regret analysis in single-agent RL (see e.g. \cite{jaksch2010near,osband2013more}, \cite{gopalan2015thompson,azar2017minimax,ouyang2017learning,jin2018q,zhang2019regret,zanette2019tighter,wei2020model,wei2021learning,chen2021implicit,jafarnia2021pomdp,jafarnia2021online} and references therein). Extending these results to the SGs is non-trivial since the actions of the opponent also affect the state transition and can not be controlled by the agent. We review the literature on exploration in SGs and refer the interested reader to \cite{zhang2021multi,yang2020overview} for an extensive literature review on multi-agent RL in various settings.

\paragraph{Stochastic Games.} A few recent works use self-play as a method to learn stochastic games \citep{bai2020provable,bai2020near,liu2021sharp,chen2021almost}. However, self-play requires controlling both the agent and the opponent and cannot be applied in the online setting where the agent plays against an arbitrary opponent. All of these works consider the setting of finite-horizon SG where the interaction of the players and the environment terminates after a fixed number of steps.

In the online setting where the opponent is \textit{arbitrary}, \cite{xie2020learning,jin2021power} achieve a regret bound of $\otil(\sqrt{T})$ in the finite-horizon SGs with linear and general function approximation, respectively. However, in the applications where the interaction between the players and the environment is non-stopping (e.g., stock trading), the infinite-horizon SG is more suitable. Lack of a fixed horizon in this setting makes the problem more challenging. This is since the backward induction, a technique that is widely used in the finite-horizon, is not applicable in the infinite-horizon setting. A recent paper on posterior sampling-based approaches to finite-horizon stochastic games is \cite{zhou2019posterior}.

In the infinite-horizon setting, the primary work of \cite{brafman2002r} who proposed R-max algorithm does not consider regret. A special case of online learning in general-sum games is studied by \cite{digiovanni2021thompson} where the opponent is allowed to switch its stationary policy a limited number of times. They achieve a regret bound of $\otil({\ell + \sqrt{\ell T}})$ via posterior sampling, where $\ell$ is the number of switches. Their result is not directly comparable to ours because their definition of regret is different. Moreover, they assume the transition kernel is known and the opponent adopts stationary policies. To the best of our knowledge, the only existing algorithm that considers online learning against an arbitrary opponent in the infinite-horizon average-reward SG is \ucsg \citep{wei2017online}. 
\paragraph{Comparison with \ucsg \citep{wei2017online}.} Our work is closely related to \ucsg, however clear distinctions exist in the result, the algorithm, and the technical contribution:
\begin{itemize}
\item \ucsg achieves a regret bound of $\otil(\sqrt[3]{DS^2AT^2})$ under the finite-diameter assumption  (i.e., for any two states and every stationary randomized policy of the opponent, there \textit{exists} a stationary randomized policy for the agent to move from one state to the other in finite expected time). Under the much stronger ergodicity assumption (i.e., for any two states and \textit{every} stationary randomized policy of the agent and the opponent, it is possible to move from one state to the other in finite expected time), \ucsg obtains a regret bound of $\otil(DS\sqrt{AT})$. Note that the ergodicity assumption greatly alleviates the challenge in exploration. Our algorithm significantly improves this result and achieves a regret bound of $\otil(HS\sqrt{AT})$ under the finite-diameter assumption.
\item \ucsg is an optimism-based algorithm inspired by \cite{jaksch2010near} and requires the complicated maximin extended value iteration. Our algorithm, however, is the first posterior sampling-based algorithm in SGs, leveraging the ideas of \cite{ouyang2017learning} in MDPs, and is much simpler both in the algorithm and the analysis. {Note that considering randomized policies in SGs (compared to MDPs) brings some challenges in applying the concentration bounds because of the continuous space of randomized policies. However, we handle this by simply using the tower property of conditional expectation which allows us to replace the continuous space of randomized policies with the finite space of actions.}
\item From the analysis perspective, under the finite-diameter assumption, \ucsg uses a sequence of finite-horizon SGs to approximate the average-reward SG and that leads to the sub-optimal regret bound of $\order(T^{2/3})$. Our analysis avoids the finite-horizon approximation by directly using the Bellman equation in the infinite-horizon SG and achieves near-optimal regret bound.
\end{itemize}

{We note that the main challenge in online learning in a Stochastic Game (SG) is the opponent's non-stationarity and uncontrollability.  \cite{wei2017online} developed a technique to replace the opponent's non-stationary policy with a stationary one in their analysis leading to a very complicated analysis and sub-optimal regret bound. We significantly simplify the analysis and improve the final regret bound with a novel technique in which we replace the opponent's policy with arbitrary distribution over actions.}


\section{PRELIMINARIES}
\label{sec: preliminaries}
Let $M = (\mathcal{S, A}, r, \theta)$ be a stochastic zero-sum game where $\mathcal{S}$ is  the state space, $\mathcal{A = A}^1 \times \mathcal{A}^2$ is the joint action space, $r: \mathcal{S \times A}^1 \times \mathcal{A}^2 \to [-1, 0]$ is the reward function and $\theta: \mathcal{S \times S \times A}^1 \times \mathcal{A}^2$ represents the transition kernel such that $\theta(s' | s, a^1, a^2) = \P(s_{t+1}=s' | s_t=s, a_t^1=a^1, a_t^2=a^2)$ where $s_t \in \S, a_t^1 \in \A^1, a_t^2 \in \A^2$ are the state, the agent and the opponent's actions at time $t=1, 2, 3, \cdots$, respectively. We assume that $\mathcal{S, A}$ are finite sets with size $S = |\mathcal{S}|, A=|\mathcal{A}|$.

The game starts at some initial state $s_1$. At time $t = 1, 2, 3, \cdots$, the players observe state $s_t$ and take actions $a_t^1, a_t^2$. The agent (maximizer) receives reward $r(s_t, a_t^1, a_t^2)$ from the opponent (minimizer). Then, the state evolves to $s_{t+1}$ according to the probability distribution $\theta(\cdot | s_t, a_t^1, a_t^2)$. The goal of the agent is to maximize its cumulative reward while the opponent tries to minimize it. For the ease of notation, we denote $a := (a^1, a^2)$ and $a_t := (a_t^1, a_t^2)$ and accordingly $r(s_t, a_t^1, a_t^2), \theta(\cdot | s_t, a_t^1, a_t^2)$ will be denoted by $r(s_t, a_t)$ and $\theta(\cdot | s_t, a_t)$, respectively.

The players' actions are assumed to depend on the history. Namely, denote by $\pi_t^1$ (resp. $\pi_t^2$) the mappings from the history $h_t = (s_1, a_1, \cdots, s_{t-1}, a_{t-1}, s_t)$ to the probability distributions over $\calA_1$ (resp. $\calA_2$). Let $\pi^1 := (\pi_1^1, \pi_2^1, \cdots)$ (resp. $\pi^2 := (\pi_1^2, \pi_2^2, \cdots)$) be the sequence of history-dependent randomized policies whose class is denoted by $\Pi^\text{HR}$. In the case that $\pi_t^1$ (resp. $\pi_t^2$) is independent of time (stationary randomized policies), we remove the subscript $t$ and with abuse of notation denote $\pi^1 := (\pi^1, \pi^1, \cdots)$ (resp. $\pi^2 := (\pi^2, \pi^2, \cdots)$). The class of stationary randomized policies is denoted by $\Pi^\text{SR}$.

For the ease of presentation, we introduce a few notations. Let $A^1 = |\mathcal{A}^1|$, $A^2 = |\mathcal{A}^2|$ denote the size of the action spaces.  For an integer $k \geq 1$, denote by $\Delta_k$ the probability simplex of dimension $k$.  Let $q^1 \in \Delta_{A^1}$ and $q^2 \in \Delta_{A^2}$. With abuse of notation, let $r(s, q^1, q^2) := \E_{a^1 \sim q^1, a^2 \sim q^2} [r(s, a^1, a^2)]$ and $\theta(s'|s, q^1, q^2) := \E_{a^1 \sim q^1, a^2 \sim q^2} [\theta(s' | s, a^1, a^2)]$. 

To achieve a low regret algorithm, it is necessary to assume that all the states are accessible by the agent under some policy. In the special case of MDPs, this is stated by the notion of ``weakly communication'' (or ``finite diameter" \citep{jaksch2010near}) and is known to be the minimal assumption to achieve sub-linear regret \citep{bartlett2009regal}. The following assumption generalizes this notion to the stochastic games.
\begin{assumption}
\label{ass: finite diameter}
\textbf{(Finite Diameter)} There exists $D \geq 0$ such that for any stationary randomized policy $\pi^2 \in \Pi^\text{SR}$ of the opponent and any $s, s' \in \mathcal{S \times S}$, there exists a stationary randomized policy $\pi^1 \in \Pi^\text{SR}$ of the agent, such that the expected time of reaching $s'$ starting from $s$ under policy $\pi = (\pi^1, \pi^2)$ does not exceed $D$, i.e.,
\begin{align*}
\max_{s, s'} \max_{\pi^2 \in \Pi^\text{SR}} \min_{\pi^1 \in \Pi^\text{SR}} T_{s \to s'}^\pi \leq D,
\end{align*}
where $T_{s \to s'}^\pi$ is the expected time of reaching $s'$ starting from $s$ under policy $\pi = (\pi^1, \pi^2)$.
\end{assumption}
This assumption was first introduced by \cite{federgruen1978n} and is essential to achieve low regret algorithms in the adversarial setting \citep{wei2017online}. To see this, suppose that the opponent has a way to lock the agent in a ``bad" state. In the initial stages of the game when the agent has limited environment knowledge, it may not be possible to avoid such a state and linear regret is unavoidable. This assumption states that regardless of the strategy used by the opponent, the agent has a way to recover from such bad states.

For a zero-sum matrix game with matrix $G$ of size $m \times n$, the game value is denoted by $\val(G) = \max_{p \in \Delta_m}\min_{q \in \Delta_n}p^TGq = \min_{q \in \Delta_n} \max_{p \in \Delta_m}p^TGq$. Moreover, the Nash equilibrium $p^* \in \Delta_m, q^* \in \Delta_n$ always exists \citep{nash1950equilibrium}. For SGs, under Assumption \ref{ass: finite diameter}, \citet{federgruen1978n,wei2017online} prove that there exist unique $J(\theta) \in \mathbb{R}$ and unique (upto an additive constant) function $v(\cdot, \theta): \mathcal{S} \to \mathbb{R}$ that satisfy the Bellman equation, i.e., for all $s \in \calS$,
\begin{align}
\label{eq: Bellman val}
J(\theta) + v(s, \theta) &= \val \left\{r(s, \cdot, \cdot) + \sum_{s'}\theta(s' | s, \cdot, \cdot)v(s', \theta)\right\}.
\end{align}
In particular, the Nash equilibrium of the right hand side for each $s \in \mathcal{S}$ yields maximin stationary policies $\pi^* = (\pi^{1*}, \pi^{2*})$ such that
\begin{align}
J(\theta) + v(s, \theta) &= \max_{q^1 \in \Delta_{A^1}} \Big\{r(s, q^1, \pi^{2*}(\cdot|s)) \nonumber \\
&+ \sum_{s'}\theta(s' | s, q^1, \pi^{2*}(\cdot|s))v(s', \theta)\Big\}, \\
J(\theta) + v(s, \theta) &= \min_{q^2 \in \Delta_{A^2}} \Big\{r(s, \pi^{1*}(\cdot|s), q^{2}) \nonumber \\
&+ \sum_{s'}\theta(s' | s, \pi^{1*}(\cdot|s), q^{2})v(s', \theta)\Big\}. \label{eq: nash 3}
\end{align}
Moreover, $J(\theta)$ is the maximin average reward obtained by the agent and is independent of the initial state $s_1$, i.e.,
\begin{align*}
J(\theta) &= \sup_{\pi^1 \in \Pi^\text{HR}}\inf_{\pi^2 \in \Pi^\text{HR}} \liminf_{T \to \infty} \frac{1}{T}\E \Bigg[\sum_{t=1}^T r(s_t, a_t) | s_1 = s \Bigg],
\end{align*}
where $a_t = (a_t^1, a_t^2)$ and $a_t^1 \sim \pi^{1}_t(\cdot|h_t)$ and $a_t^2 \sim \pi^{2}_t(\cdot|h_t)$. Note that $J(\theta) \in [-1, 0]$ because the range of the reward function is $[-1, 0]$. Define the \textit{span} of the stochastic game with transition kernel $\theta$ as the span of the corresponding value function $v$, i.e., $\spn(\theta) := \max_sv(s, \theta) - \min_sv(s, \theta)$. We restrict our attention to stochastic games whose transition kernel $\theta$ satisfies Assumption \ref{ass: finite diameter} and $\spn(\theta) \leq H$ where $H$ is a known scalar. {This constant is not used explicitly in the algorithm we propose but is implicit since all transition kernels we allow have bias-span bounded by $H$.} Let $\Omega_*$ denote the set of all such $\theta$. Moreover, observe that if $v$ satisfies the Bellman equation, $v + c$ also satisfies the Bellman equation for any scalar $c$. Thus, without loss of generality, we can assume that $0 \leq v(s, \theta) \leq H$ for all $s \in \calS$ and $\theta \in \Omega_*$.

\paragraph{Stationary Randomized Opponent.} We first consider the special case where the opponent follows a fixed unknown stationary randomized policy $\pi^2$. In that case, the agent can consider the opponent as part of the environment and define a new environment with reward and transition kernel
\begin{align*}
r^{\pi^2}(s, a^1) &:= r(s, a^1, \pi^2(s)), \\
\theta^{\pi^2}(s'|s, a^1) &:= \theta(s'|s, a^1, \pi^2(s)).
\end{align*}
Since the new environment is stationary, the agent can use any standard single-agent RL algorithm. For example, applying \texttt{TSDE} algorithm \citep{ouyang2017learning} yields a regret bound of $\otil(DS\sqrt{A^1T})$.\footnote{The original bound in \citet{ouyang2017learning} is $\otil(HS\sqrt{A^1T})$ where $H$ is an upper bound on the span of the relative value function. Assumption~\ref{ass: finite diameter} implies that the diameter and thus the span of the relative value function of the induced MDP is upper bounded by $D$.} The rest of the paper considers the more general case where the opponent can take any time-adaptive randomized policy.

\paragraph{Time-adaptive Randomized Opponent.} The focus of this paper is on the case where the agent plays a stochastic game $(\S, \A, r, \theta_*)$ against an opponent who can take time-adaptive policies. We assume that the opponent knows the history of states and actions and can play time-adaptive history-dependent policies. Recall that the state of such policies is denoted by $\Pi^\text{HR}$. Considering the opponent as part of the environment in this case results in a time-varying environment and, therefore, standard single-agent no-regret algorithms are not applicable. $\S, \A$ and $r$ are completely known to the agent. However, the transition kernel $\theta_*$ is unknown. In the beginning of the game, $\theta_*$ is drawn from an initial distribution $\mu_1$ and is then fixed. We assume that the support of $\mu_1$ is a subset of $\Omega_*$. The performance of the agent is then measured with the notion of regret defined as 
\begin{align*}
R_T := \sup_{\pi^2 \in \Pi^\text{HR}} \E \left[\sum_{t=1}^T (J(\theta_*) - r(s_t, a_t))\right],
\end{align*}
where $a_t^2 \sim \pi_t^2(\cdot|h_t)$. Here the expectation is with respect to the prior distribution $\mu_1$, randomized algorithm and the randomness in the state transition. Note that the regret guarantee is against an arbitrary opponent who can change its policy at each time step and has the perfect knowledge of the history of the states and actions. The only hidden information from the opponent is the realization of the agent's current action (which will be revealed after both players have chosen their actions). We note that self-play and the case when the agent and the opponent use the same learning algorithm are two special cases of the scenario considered here.


\section{POSTERIOR SAMPLING FOR STOCHASTIC GAMES}
\label{sec: algorithm}
In this section, we propose Posterior Sampling algorithm for Zero-sum SGs (\pssg). The agent maintains the posterior distribution $\mu_t$ on parameter $\theta_*$. More precisely, the learning algorithm receives an initial distribution $\mu_1$ as the input and updates the posterior distribution upon observing the new state according to
\begin{align}
\label{eq: posterior update}
\mu_{t+1}(d\theta) \propto \theta(s_{t+1} | s_t, a_t) \mu_{t}(d\theta).
\end{align}

\pssg proceeds in episodes. Let $t_k, T_k$ denote the start time and the length of episode $k$, respectively. In the beginning of each episode, the agent draws a sample of the transition kernel from the posterior distribution $\mu_{t_k}$. The maximin strategy is then derived for the sampled transition kernel according to \eqref{eq: Bellman val} and used by the agent during the episode. Let $N_t(s, a)$ be the number of visits to state-action pair $(s, a) = (s, a^1, a^2)$ before time $t$, i.e.,
\begin{align*}
N_t(s, a) = \sum_{\tau=1}^{t-1} \1(s_\tau=s, a_\tau=a).
\end{align*}
As described in Algorithm \ref{alg: posterior sampling}, a new episode starts if $t > t_k + T_{k-1}$ or $N_t(s, a) > 2 N_{t_k}(s, a)$ for some $(s, a)$. The first criterion, $t > t_k + T_{k-1}$, states that the length of the episode grows at most by 1 if the other criterion is not triggered. This ensures that $T_k \leq T_{k-1} + 1$ for all $k$.
The second criterion is triggered if the number of visits to a state-action pair is doubled. These stopping criteria balance the trade-off between exploration and exploitation. In the beginning of the game, the episodes are short to motivate exploration since the agent is uncertain about the underlying environment. As the game proceeds, the episodes grow to exploit the information gathered about the environment. These stopping criteria are the same as those used in MDPs \citep{ouyang2017learning}.

\begin{algorithm}[t]
\caption{\textsc{\pssg}}
\label{alg: posterior sampling}
\textbf{Input: } $\mu_1$\\
\textbf{Initialization: }$t \gets 1, t_1 \gets 0$\\ 
\For{ episodes $k=1, 2, \cdots$}{
	$T_{k-1} \gets t - t_k$\\
	$t_k \gets t$\\
	Generate $\theta_k \sim \mu_{t_k}$ and compute $\pi_k^1(\cdot)$ using \eqref{eq: Bellman val}\\
	\While{$t \leq t_k + T_{k-1}$ and $N_t(s, a) \leq 2 N_{t_k}(s, a)$ for all $(s, a) \in \S \times \A$}{
	Choose action $a_t^1 \sim \pi_k^1(\cdot | s_t)$ and observe $a_t^2$, $s_{t+1}$\\
	Update $\mu_{t+1}$ according to \eqref{eq: posterior update}\\
	$t \gets t+1$	
	}
} 
\end{algorithm}

Algorithm \ref{alg: posterior sampling} can achieve regret bound of $\otil(HS\sqrt{AT})$. This result improves upon the previous best known result of \ucsg algorithm which achieves $\otil(\sqrt[3]{DS^2AT^2})$ under the same assumption \citep{wei2017online}.

\begin{theorem}
\label{thm: regret}
Under Assumption \ref{ass: finite diameter}, Algorithm \ref{alg: posterior sampling} can achieve regret bound of
\begin{align}
R_T &\leq (H+1)\sqrt{2SAT\log T} + H \nonumber \\
&+ H \left(SA + 2\sqrt{SAT}\right)\sqrt{224S\log (2AT)}.
\end{align}
\end{theorem}


\section{ANALYSIS}
\label{sec: analysis}
In this section, we provide the proof of Theorem \ref{thm: regret}. A central observation in our analysis is that in the beginning of each episode, $\theta_*$ and $\theta_k$ are identically distributed conditioned on the history. This key property of posterior sampling relates quantities that depend on the unknown $\theta_*$ to those of the sampled $\theta_k$ which is fully observed by the agent. Posterior sampling ensures that if $t_k$ is a stopping time, for any measurable function $f$ and any $h_{t_k}$-measurable random variable $X$, $\E[f(\theta_*, X) | h_{t_k}] = \E[f(\theta_k, X) | h_{t_k}]$ \citep{ouyang2017learning,osband2013more}.

The key challenge in the analysis of stochastic games is that the opponent is also making decisions. If the opponent follows a fixed stationary policy, it can be considered as part of the environment and thus the SG reduces to an MDP. However, in the case that the opponent uses a dynamic history-dependent policy during the learning phase of the agent, this reduction is not possible. The key lemma in our analysis is Lemma \ref{lem:  bound second} which overcomes this difficulty through the Bellman equation for the SG.
\subsection{Proof of Theorem \ref{thm: regret}}
Let $K_T := \max \{k: t_k \leq T \}$ be the number of episodes until time $T$ and define $t_{K_T+1} = T+1$. Recall that $R_T = \sup_{\pi^2 \in \Pi^\text{HR}}R_T(\pi^2)$ where
\begin{align}
R_T(\pi^2) = \E \left[TJ(\theta_*) - \sum_{t=1}^T r(s_t, a_t)\right].
\end{align}
Let $\pi^2 \in \Pi^\text{HR}$ be an arbitrary history-dependent randomized strategy followed by the opponent. We start by decomposing the regret into two terms
\begin{align}
R_T(\pi^2) &= \E \left[TJ(\theta_*) - \sum_{t=1}^T r(s_t, a_t)\right]  \nonumber \\
&= \E \left[TJ(\theta_*) - \sum_{k=1}^{K_T}\sum_{t=t_k}^{t_{k+1}-1}J(\theta_k)\right] \nonumber \\
&+ \E \left[\sum_{k=1}^{K_T}\sum_{t=t_k}^{t_{k+1}-1} (J(\theta_k) - r(s_t, a_t))\right] \label{eq: regret decomposition}.
\end{align}
Lemma \ref{lem: bound first} uses the property of posterior sampling to bound the first term. The second term is  handled by combining the Bellman equation, concentration inequalities and the property of posterior sampling as detailed in Lemma \ref{lem: bound second}. Finally, Lemma~\ref{lem: number of episodes} bounds the number of episodes and completes the proof.
\begin{lemma}
\label{lem: bound first}
The first term of \eqref{eq: regret decomposition} can be bounded by
\begin{align*}
 \E \left[TJ(\theta_*) - \sum_{k=1}^{K_T}\sum_{t=t_k}^{t_{k+1}-1}J(\theta_k)\right] \leq \E[K_T]
\end{align*}
\end{lemma}
\begin{proof}
\begin{align}
\label{eq: proof sum J}
\sum_{k=1}^{K_T}\sum_{t=t_k}^{t_{k+1}-1}J(\theta_k) &= \sum_{k=1}^{K_T}T_kJ(\theta_k) = \sum_{k=1}^{\infty}\1(t_k \leq T)T_kJ(\theta_k) \nonumber \\
&\geq \sum_{k=1}^{\infty}\1(t_k \leq T)(T_{k-1} + 1)J(\theta_k)
\end{align}
where the last inequality is by the fact that $J(\theta_k) \leq 0$ and $T_k \leq T_{k-1} + 1$ due to the first stopping criterion. Now, note that $t_k$ is a stopping time and $\1(t_k \leq T)$ and $T_{k-1}$ are $h_{t_k}$-measurable random variables. Thus, by the property of posterior sampling and monotone convergence theorem,
\begin{align*}
&\E \left[\sum_{k=1}^{\infty}\1(t_k \leq T)(T_{k-1} + 1)J(\theta_k)  \suchthat h_{t_k}\right] \\
&= \sum_{k=1}^{\infty}\E \left[\1(t_k \leq T)(T_{k-1} + 1)J(\theta_k)  \suchthat h_{t_k}\right] \\
&= \sum_{k=1}^{\infty}\E \left[\1(t_k \leq T)(T_{k-1} + 1)J(\theta_*)  \suchthat h_{t_k}\right] \\
&= \E \left[\sum_{k=1}^{\infty}\1(t_k \leq T)(T_{k-1} + 1)J(\theta_*)  \suchthat h_{t_k}\right] \\
&\geq \E \left[\sum_{k=1}^{K_T}(T_{k-1} + 1)J(\theta_*) \suchthat h_{t_k}\right].
\end{align*}
Taking another expectation from both sides and using the tower property, we have
\begin{align*}
&\E \left[\sum_{k=1}^{\infty}\1(t_k \leq T)(T_{k-1} + 1)J(\theta_k) \right] \\
&\geq \E \left[\sum_{k=1}^{K_T}(T_{k-1} + 1)J(\theta_*) \right].
\end{align*}
Replacing this in \eqref{eq: proof sum J} implies that
\begin{align*}
&\E \left[TJ(\theta_*) - \sum_{k=1}^{K_T}\sum_{t=t_k}^{t_{k+1}-1}J(\theta_k)\right] \\
 &\leq \E\left[(T - \sum_{k=1}^{K_T}T_{k-1})J(\theta_*)\right] - \E[K_TJ(\theta_*)] \leq \E[K_T].
\end{align*}
The last inequality is by the fact that $T - \sum_{k=1}^{K_T}T_{k-1} \leq 0$ and $J(\theta_*) \in [-1, 0]$.
\end{proof}

\begin{lemma}
\label{lem:  bound second}
The second term of \eqref{eq: regret decomposition} can be bounded by
\begin{align*}
&\E \left[\sum_{k=1}^{K_T}\sum_{t=t_k}^{t_{k+1}-1} (J(\theta_k) - r(s_t, a_t))\right] \leq H\E[K_T] + H \\
&+ \sqrt{224S\log (2AT)}(HSA + 2H\sqrt{SAT}).
\end{align*}
\end{lemma}
\begin{proof}
The policy $\pi^1_k$ used by the agent at episode $k$ is the solution of the Nash equilibrium in \eqref{eq: Bellman val}. Thus, for $t_k \leq t \leq t_{k+1}-1$ and any $s \in \S$, \eqref{eq: nash 3} implies that
\begin{align*}
&J(\theta_k) + v(s, \theta_k) \\
&\leq r(s, \pi^1_k(\cdot|s), q^{2}) + \sum_{s'}\theta_k(s' | s, \pi^1_k(\cdot|s), q^{2})v(s', \theta_k),
\end{align*}
for any distribution $q^2 \in \Delta_{A^2}$. Let $\pi^2 = (\pi_1^2, \pi_2^2, \cdots) \in \Pi^\text{HR}$ be an arbitrary history-dependent randomized strategy for the opponent. Note that for any $t \geq 1$, $\pi_t^2$ is $h_t$-measurable. Replacing $s$ by $s_t$ and $q^2$ by $\pi^2_t(\cdot|h_t)$ implies that
\begin{align*}
&J(\theta_k) - r(s_t, \pi^1_k(\cdot|s_t), \pi^2_t(\cdot|h_t)) \\
&\leq \sum_{s'}\theta_k(s' | s_t, \pi^1_k(\cdot|s_t), \pi^{2}_t(\cdot|h_t))v(s', \theta_k) - v(s_t, \theta_k).
\end{align*}
Adding and subtracting $v(s_{t+1}, \theta_k)$ to the right hand side and summing over time steps within episode $k$ implies that
\begin{align}
&\sum_{t=t_k}^{t_{k+1}-1} \left(J(\theta_k) - r(s_t, \pi^1_k(\cdot|s_t), \pi^2_t(\cdot|h_t))\right) \nonumber \\
&\leq \sum_{t=t_k}^{t_{k+1}-1}\Bigg(\sum_{s'}\theta_k(s' | s_t, \pi^1_k(\cdot|s_t), \pi^{2}_t(\cdot|h_t))v(s', \theta_k)\nonumber \\
&\qquad \qquad \qquad - v(s_{t+1}, \theta_k)\Bigg) \nonumber \\
&\qquad+ \sum_{t=t_k}^{t_{k+1}-1}\left(v(s_{t+1}, \theta_k) - v(s_t, \theta_k)\right). \label{eq: pf lem 2 tmp 2}
\end{align}
The second term on the right hand side of \eqref{eq: pf lem 2 tmp 2} telescopes and can be bounded as
\begin{align}
\sum_{t=t_k}^{t_{k+1}-1}\left(v(s_{t+1}, \theta_k) - v(s_t, \theta_k)\right) &= v(s_{t_{k+1}}, \theta_k) - v(s_{t_k}, \theta_k) \nonumber \\
&\leq H, \label{eq: pf lem 2 tmp 3}
\end{align}
where the last inequality is by the fact that $\theta_k$ is chosen from the posterior distribution whose support is a subset of $\Omega_*$. Substituting \eqref{eq: pf lem 2 tmp 3} in \eqref{eq: pf lem 2 tmp 2}, summing over episodes, and taking expectation implies that
\begin{align*}
&\E \left[\sum_{k=1}^{K_T}\sum_{t=t_k}^{t_{k+1}-1} (J(\theta_k) - r(s_t, a_t))\right] \\
&= \E\left[\sum_{k=1}^{K_T}\sum_{t=t_k}^{t_{k+1}-1} \left(J(\theta_k) - r(s_t, \pi^1_k(\cdot|s_t), \pi^2_t(\cdot|h_t))\right) \right] \\
&\leq H\E[K_T] + \E \Bigg[\sum_{k=1}^{K_T}\sum_{t=t_k}^{t_{k+1}-1}\\
&\sum_{s'}\theta_k(s' | s_t, \pi^1_k(\cdot|s_t), \pi^{2}_t(\cdot|h_t))v(s', \theta_k) - v(s_{t+1}, \theta_k)\Bigg].
\end{align*}
We proceed to bound the last term on the right hand side of the above inequality. Before proceeding note that if  $k(t)$ denotes the episode at time $t$, a random variable. Then,  for any $t \geq 1$ and $s' \in \mathcal{S}$,
\begin{align*}
\mathbb{E}\Big[\theta_{k(t)}(s'|s_t, a_t^1, a_t^2) \big| h_{t}, \theta_{k(t)}\Big] = \\ 
\theta_{k(t)}\Big(s' \big| s_t, \pi_{k(t)}^1(\cdot|s_t), \pi_{t}^2(\cdot|h_{t})\Big), \qquad (*)
\end{align*}
because $a_t^1 \sim \pi_{k(t)}^1(\cdot|s_t)$ and $a_t^2 \sim \pi_{t}^2(\cdot|h_{t})$.
Now,
\begin{align} 
&\E \Bigg[\sum_{k=1}^{K_T}\sum_{t=t_k}^{t_{k+1}-1} \nonumber \\
&\sum_{s'}\theta_k(s' | s_t, \pi^1_k(\cdot|s_t), \pi^{2}_t(\cdot|h_t))v(s', \theta_k) - v(s_{t+1}, \theta_k)\Bigg] \nonumber \\
&=\E \Bigg[\sum_{k=1}^{K_T}\sum_{t=t_k}^{t_{k+1}-1} \nonumber \\
&\sum_{s'}\theta_k(s' | s_t, a_t^1, a^2_t)v(s', \theta_k) - v(s_{t+1}, \theta_k)\Bigg]= \nonumber \\
& \E \left[\sum_{k=1}^{K_T}\sum_{t=t_k}^{t_{k+1}-1}\sum_{s'}\left[\theta_k(s' | s_t, a_t) - \theta_*(s' | s_t, a_t)\right]v(s', \theta_k)\right] \nonumber \\
&\leq H \E \left[\sum_{k=1}^{K_T}\sum_{t=t_k}^{t_{k+1}-1}\sum_{s'}\suchthat\theta_k(s' | s_t, a_t) - \theta_*(s' | s_t, a_t)\suchthat\right] \label{eq: pf lem 2 tmp 4}
\end{align}
To bound the inner summation, similar to \cite{ouyang2017learning,jaksch2010near}, we define a confidence set $\mathcal{C}_k$ around the empirical transition kernel $\hat{\theta}_k(s' | s, a) := \frac{N_{t_k}(s', s, a)}{N_{t_k}(s, a)}$. Here $N_{t_k}(s', s, a) := \sum_{t=1}^{t_k-1}\1(s_t=s, a_t=a, s_{t+1}=s')$ is the number of visits to state-action pair $(s, a)$ whose next state is $s'$. The confidence set $\mathcal{C}_k$ is defined as $\mathcal{C}_k :=$
\begin{align*}
\{\theta : \sum_{s'}|\theta(s'|s, a) - \hat{\theta}_k(s'|s, a)| \leq b_k(s, a) \quad \forall s, a, s'\},
\end{align*}
where $b_k(s, a) := \sqrt{\frac{14S\log (2At_kT)}{\max\{1, N_{t_k}(s, a)\}}}$. \cite{weissman2003inequalities} shows that the true transition kernel $\theta_*$ belongs to $\mathcal{C}_k$ with high probability. We use this fact to show concentration of $\hat{\theta}_k$ around $\theta_*$. Concentration of $\hat{\theta}_k$ around $\theta_k$ is then followed by the property of posterior sampling. More precisely, we can write
\begin{align*}
&\sum_{s'}|\theta_k(s' | s_t, a_t) - \theta_*(s' | s_t, a_t)|\\
&\leq \sum_{s'}|\theta_k(s' | s_t, a_t) - \hat{\theta}_k(s'|s_t, a_t)| \\
&\qquad+ \sum_{s'}| \theta_*(s' | s_t, a_t) - \hat{\theta}_k(s'|s_t, a_t)| \\
&\leq 2b_k(s_t, a_t) + 2\left(\1(\theta_k \notin \mathcal{C}_k) + \1(\theta_* \notin \mathcal{C}_k)\right).
\end{align*}
Substituting the inner sum of \eqref{eq: pf lem 2 tmp 4} with this upper bound implies
\begin{align}
\label{eq: pf lem 2 tmp 5}
&H \E \left[\sum_{k=1}^{K_T}\sum_{t=t_k}^{t_{k+1}-1}\sum_{s'}\suchthat\theta_k(s' | s_t, a_t) - \theta_*(s' | s_t, a_t)\suchthat\right] \nonumber \\
&\leq 2H \left\{ \sum_{k=1}^{K_T}\sum_{t=t_k}^{t_{k+1}-1} b_k(s_t, a_t)\right\} \nonumber \\
&\qquad+ 2H \E\left[\sum_{k=1}^{K_T}T_k\{\1(\theta_k \notin \mathcal{C}_k) + \1(\theta_* \notin \mathcal{C}_k)\}\right].
\end{align}
The first term on the right hand side of \eqref{eq: pf lem 2 tmp 5} can be bounded as
\begin{align}
\label{eq: pf lem 2 tmp 6}
\sum_{k=1}^{K_T}\sum_{t=t_k}^{t_{k+1}-1} &b_k(s_t, a_t) = \sum_{k=1}^{K_T}\sum_{t=t_k}^{t_{k+1}-1} \sqrt{\frac{14S\log (2At_kT)}{\max\{1, N_{t_k}(s_t, a_t)\}}} \nonumber \\
&\leq \sum_{k=1}^{K_T}\sum_{t=t_k}^{t_{k+1}-1}\sqrt{\frac{28S\log (2AT^2)}{\max\{1, N_{t}(s_t, a_t)\}}} \nonumber \\
&= \sum_{t=1}^T\sqrt{\frac{28S\log (2AT^2)}{\max\{1, N_{t}(s_t, a_t)\}}} \nonumber \\
&\leq \sqrt{56S\log (2AT)} (SA + 2\sqrt{SAT}),
\end{align}
where the first inequality is by the fact that $t_k \leq T$ and $N_t(s, a) \leq 2N_{t_k}(s, a)$ for all $s, a$ and the second inequality is by the following argument:
\begin{align*}
&\sum_{t=1}^T\sqrt{\frac{1}{\max\{1, N_{t}(s_t, a_t)\}}} = \sum_{t=1}^T\sum_{s, a}\frac{\1(s_t=s, a_t=a)}{\sqrt{\max\{1, N_{t}(s, a)\}}} \\
&= \sum_{s, a}\sum_{t=1}^T\frac{\1(s_t=s, a_t=a)}{\sqrt{\max\{1, N_{t}(s, a)\}}} \\
&= \sum_{s, a} \left(1 + \sum_{j=1}^{n_{T+1}(s, a)-1}\frac{1}{\sqrt{j}}\right) \\
&\leq \sum_{s, a} \left(1 + 2\sqrt{N_{T+1}(s, a)} \right) = SA + 2\sum_{s, a}\sqrt{N_{T+1}(s, a)} \\
&\leq SA + 2\sqrt{SA \sum_{s, a}N_{T+1}(s, a)} = SA + 2\sqrt{SAT},
\end{align*}
where the last inequality is by Cauchy-Schwarz and the last equality is by the fact that $\sum_{s, a}N_{T+1}(s, a) = T$.
To bound the second term on the right hand side of \eqref{eq: pf lem 2 tmp 5}, we can write
\begin{align*}
&\E\left[\sum_{k=1}^{K_T}T_k\{\1(\theta_k \notin \mathcal{C}_k) + \1(\theta_* \notin \mathcal{C}_k)\}\right] \\
&\leq \E\left[\sum_{k=1}^{\infty}T\{\1(\theta_k \notin \mathcal{C}_k) + \1(\theta_* \notin \mathcal{C}_k)\}\right] \\
&= T\sum_{k=1}^{\infty} \E\left[\1(\theta_k \notin \mathcal{C}_k) + \1(\theta_* \notin \mathcal{C}_k)\right] \\
&= 2T\sum_{k=1}^{\infty} \E\left[\1(\theta_* \notin \mathcal{C}_k)\right] = 2T\sum_{k=1}^{\infty} \P(\theta_* \notin \mathcal{C}_k),
\end{align*}
where the second equality is by the property of Posterior Sampling since $\mathcal{C}_k$ is $\calF_{t_k}$-measurable. Note that $\P(\theta_* \notin \mathcal{C}_k) \leq \frac{1}{15Tt_k^6}$ (Lemma 17 of \cite{jaksch2010near}). Thus,
\begin{align}
\label{eq: pf lem 2 tmp 7}
2T\sum_{k=1}^{\infty} \P(\theta_* \notin \mathcal{C}_k) = \frac{2}{15}\sum_{k=1}^\infty \frac{1}{t_k^6} \leq \frac{2}{15}\sum_{k=1}^\infty \frac{1}{k^6} \leq \frac{1}{2}.
\end{align}
Combining \eqref{eq: pf lem 2 tmp 6} and \eqref{eq: pf lem 2 tmp 7} in \eqref{eq: pf lem 2 tmp 5} completes the proof.
\end{proof}


\begin{remark}
We note that the bias span of $v$ is  bounded by $H$ because we are using the bias span of $v$ from the minimax Bellman equation (not from the Bellman equation of the player's policies.) This key observation allows us to handle an adversarial opponent and improve the bound of prior work \citep{wei2017online} with a much simpler analysis.   
\end{remark}

It remains to bound the number of episodes. The following lemma completes the proof of Theorem \ref{thm: regret}.
\begin{lemma}
\label{lem: number of episodes}
The number of episodes can be bounded by
\begin{align*}
K_T \leq \sqrt{2SAT\log T}.
\end{align*}
\end{lemma}

\begin{proof}
Define macro episodes with start times $t_{m_i}$ as $t_{m_1} = t_1$ and for $i \geq 2$, $t_{m_i} :=$
\begin{align*}
\min\{t_k > t_{m_{i-1}} : n_{t_k}(s, a) > 2n_{t_{k-1}}(s, a), \text{for some } (s, a)\}.
\end{align*}
Note that $t_{m_i}$ is the start time of the $i$th macro episode and corresponds to the $i$th start time that an episode triggers with the second stopping criterion in Algorithm~\ref{alg: posterior sampling}. Denote by $M_T$ the number of macro episodes by time $T$ and let $m_{M_T+1} = K_T + 1$.

Let $\tilde T_i$ be the length of the $i$th macro episode. We can write $\tilde T_i = \sum_{k=m_i}^{m_{i+1}-1} T_k$. All the episodes except the last one within a macro episode are started with the first criterion. Thus, for all $m_i \leq k \leq m_{i+1}-2$, $T_k = T_{k-1} + 1$, and
\begin{align*}
\tilde T_i &= \sum_{k=m_i}^{m_{i+1}-1} T_k = T_{m_{i+1}-1} + \sum_{j=1}^{m_{i+1} - m_i -1} (T_{m_i-1}+j) \\
&\geq 1 + \sum_{j=1}^{m_{i+1} - m_i -1} (1 + j) \\
&= 0.5(m_{i+1}-m_i)(m_{i+1}-m_i+1).
\end{align*}
This implies that $m_{i+1}-m_i \leq \sqrt{2\tilde T_i}$ for all $i = 1, \cdots, M_T$. Consequently,
\begin{align}
\label{eq: pf num of episodes}
K_T &= m_{M_T + 1} - 1 = \sum_{i=1}^{M_T}(m_{i+1} - m_i) \leq \sum_{i=1}^{M_T}\sqrt{2\tilde T_i} \nonumber \\
&\leq \sqrt{2M_T\sum_{i=1}^{M_T}\tilde T_i} = \sqrt{2M_TT},
\end{align}
where the last inequality is by Cauchy-Schwarz and the last equality is due to $\sum_{i=1}^{M_T}\tilde T_i = T$. Now, it suffices to prove that $M_T \leq SA\log T$. To see this, let $\calT_{s, a}$ be the episode start times that are triggered by the second stopping criterion at state-action pair $(s, a)$. That is,
\begin{align*}
\calT_{s, a} := \{t_k \leq T : n_{t_k}(s, a) > 2 n_{t_{k-1}}(s, a)\}.
\end{align*}
Since the number of visits to state-action pair $(s, a)$ is doubled at each $t_k \in \calT_{s, a}$, we claim that $|\calT_{s, a}| \leq \log n_{T+1}(s, a)$. To see this, assume by contradiction that $|\calT_{s, a}| > \log n_{T+1}(s, a) + 1$. We can write
\begin{align*}
n_{t_{K_T}}&(s, a) \geq \prod_{t_k \leq T, n_{t_{k-1}}(s, a) \geq 1} \frac{n_{t_k}(s, a)}{n_{t_{k-1}}(s, a)} \\
&\geq \prod_{t_k \in \calT_{s, a}, n_{t_{k-1}}(s, a) \geq 1} \frac{n_{t_k}(s, a)}{n_{t_{k-1}}(s, a)} \\
&> \prod_{t_k \in \calT_{s, a}, n_{t_{k-1}}(s, a) \geq 1} 2 = 2^{|\calT_{s, a}| - 1} \geq n_{T+1}(s, a),
\end{align*}
which is a contradiction. Here, the second inequality is by the fact that $n_t(s, a)$ is non-decreasing and the last inequality is by the definition of $\calT_{s, a}$. Now, we can write
\begin{align*}
M_T &= 1 + |\calT_{s, a}| \leq 1 + \sum_{s, a}\log n_{T+1}(s, a) \\
&\leq 1 + SA \log (\sum_{s, a}n_{T+1}(s, a)/SA) \\
&= 1 + SA \log (T/SA) \leq SA \log T.
\end{align*}
Here, the second inequality is by the concavity of $\log$. Replacing this inequality in \eqref{eq: pf num of episodes} completes the proof.

\end{proof}

\section{CONCLUSIONS}

We proposed \pssg, a posterior sampling algorithm that achieves Bayesian regret bound of $\otil(HS\sqrt{AT})$ in the infinite-horizon zero-sum stochastic games with average-reward criterion. No structure is imposed on the opponent's strategy. The best existing result achieves high probability regret bound of $\otil(DS\sqrt{AT})$ only under the strong ergodicity assumption. \pssg relaxes that assumption and improves the previous best known high probability regret bound of $\otil(\sqrt[3]{DS^2AT^2})$ obtained by \ucsg algorithm \citep{wei2017online} under the same finite diameter assumption. This bound is order optimal in terms of $A$ and $T$. The framework and analysis developed in this paper may be useful for designing regret-optimal algorithms based on  the optimism in face of uncertainty principle for zero-sum stochastic games.

Please note that in a game situation, it is very challenging to have an experimental setup from which we can draw meaningful conclusions since the opponent is free to do whatever they want. A direction for future work would be to assess the proposed algorithm in a systematic manner empirically. 

\bibliographystyle{abbrvnat}
\bibliography{online_rl}


\end{document}